%% file: main.tex
\documentclass{article}

\usepackage[preprint]{neurips_2026}

\usepackage[utf8]{inputenc} 
\usepackage[T1]{fontenc}    
\usepackage{hyperref}       
\usepackage{url}            
\usepackage{booktabs}       
\usepackage{amsfonts}       
\usepackage{nicefrac}       
\usepackage{microtype}      
\usepackage{xcolor}         
\usepackage{amsmath}
\usepackage{graphicx}
\usepackage{bbm}
\usepackage{amssymb}
\usepackage{amsthm}
\usepackage{wrapfig}
\usepackage[ruled,vlined]{algorithm2e}
\usepackage{amsmath,amssymb}

\usepackage[english]{babel}
\newtheorem{theorem}{Theorem}[section]

\DeclareMathOperator*{\argmax}{arg\,max} 

\title{Encoding Robust Topological Signatures for Hyperdimensional Computing}

\author{%
  Arpan Kusari\\
  University of Michigan Transportation Research Institute\\
  University of Michigan\\
  Ann Arbor, MI 48109 \\
  \texttt{kusari@umich.edu} \\
}

\begin{document}

\maketitle

\begin{abstract}
Hyperdimensional (HD) computing offers an attractive alternative to deep networks for edge learning due to its simplicity, fast prototype-based inference, and compatibility with online updates. However, standard pixel-based HD encoders are brittle: small distribution shifts such as rotation, noise, or occlusion can drastically reduce accuracy. We extract discrete topological primitives-most notably holes-from binarized shapes and pair them with rotation/translation/scale (RTS)–invariant shape signatures. Our method constructs RTS–stable descriptors for (i) the outer shape using a spatial-pyramid variant of Zernike moments and (ii) each hole using an intrinsic Fourier descriptor of its radial signature together with RTS-canonical relative geometry. Each primitive is mapped to a bipolar hypervector via randomized projection and role binding, and variable-cardinality hole sets are aggregated by permutation-invariant bundling to form a single image hypervector. To avoid over-weighting any cue, we learn nonnegative reliability weights for the Zernike and hole channels on a validation set via late fusion of cosine similarities. Experiments on MNIST and EMNIST under controlled corruptions (rotation, Gaussian noise, salt-and-pepper, cutout, zoom) show that Topology-guided HD computing substantially improves robustness compared with a naive HD baseline, maintaining high accuracy across multiple corruption families and benefiting from lightweight online training. Compared with a compact CNN trained on clean data, our method achieves competitive clean accuracy while offering markedly stronger robustness to several pixel-level corruptions, demonstrating that explicit topological structure is a practical route to robust HD representations. The code is provided \href{https://github.com/arpan-kusari/Topological-HDC}{here}.
\end{abstract}

\section{Introduction}
Hyperdimensional (HD) computing has emerged as a compelling computational analogue to neural architectures, mapping data into very high-dimensional spaces to perform tasks such as classification and regression. A central design objective of HD computing is to achieve strong performance with extreme memory and resource efficiency, while remaining robust to large-scale corruption and noise in the input \citep{kanerva2009hyperdimensional, rahimi2016robust, kleyko2023survey}. Notably, early artificial neural networks were motivated by similar constraints—e.g., the multi-layer perceptron (MLP) \citep{rosenblatt1958perceptron}—but contemporary deep learning systems have largely shifted away from these goals, trading compactness and efficiency for depth, parameter count, and substantial compute. In parallel, it is now well established that modern deep neural networks can be fragile: even small, structured perturbations may trigger severe misclassification \citep{nguyen2015deep}. Against this backdrop, HD computing offers an alternative paradigm in which noise tolerance and efficiency are preliminary considerations.

At the core of HD computing is an encoding step that maps raw inputs into a high-dimensional representation, formalized as a function $\phi: \chi \rightarrow \mathcal{H}$. For images, a straightforward (pixel-based) encoding strategy represents each pixel by combining information about its position and its intensity \citep{smets2024encoding}. Given an image of size s with intensity levels I, $\mathcal{M} := \{(p_i, I(p_i)): 0 \leq i \leq s, 0 \leq I \leq I_{max}\}$. An HD encoder then assigns hypervectors  to positions $\phi_p$ and $\phi_I$ to intensities, and aggregates them into a single image representation via binding and bundling:
\begin{equation}
    \phi_{\mathcal{M}} = \oplus\Large (\phi_p \otimes \phi_I \Large) 
\end{equation}
where the binding operator $\otimes$ couples the position and intensity hypervectors into a joint descriptor for each pixel, and the bundling operator $\oplus$ combines these descriptors across the full image.

Although this naive pixel-based encoder-and subsequent variants such as Fourier Holographic Reduced Representation \citep{verges2025learning}-can implicitly capture aspects of global structure (or ``shape”) in the sense of \cite{carlsson2020topological}, prior HD approaches have not, to the best of our knowledge, explicitly encoded topological signatures of image data. This omission is particularly relevant in computer vision, where topology has been shown to be intrinsic to natural image statistics. For example, \cite{carlsson2008local} demonstrated that naturally occurring image patches concentrate near an embedded three-dimensional manifold with the topology of a Klein bottle. 

\begin{figure}
    \centering
    \includegraphics[width=1.0\linewidth]{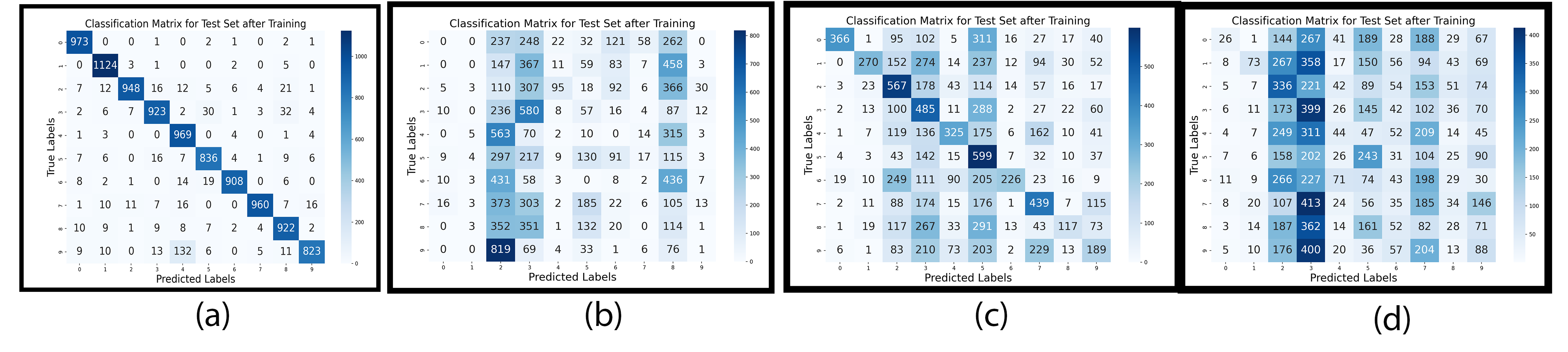}
    \caption{Confusion matrices of MNIST dataset classification after training for 20 epochs using naive HDC implementation with the following configurations: (a) Clean test dataset with accuracy of 0.95; (b) adding gaussian noise of $\sigma=0.1$ results in accuracy of 0.09; (c) salt-and-pepper noise for 10\% of pixels with accuracy of 0.65; and (d) zoom of 0.75 with accuracy of 0.15. The confusion matrices clearly show mode collapse occurring due to added noise. }
    \label{fig:naive_composite}
\end{figure}

To further motivate our case, we run experiments on the HD classification of MNIST dataset (see Fig. \ref{fig:naive_composite} for confusion matrices and descriptions). We find that while the naive HD classifier can provide a very high accuracy in the case of clean test data, the accuracy drops drastically with some corruption added (gaussian, salt-and-pepper, zooming the digits etc). 

Driven by these empirical results, we propose an HD encoding framework that explicitly represents discrete topological primitives derived from images, with the specific aim of improving robustness under substantial data corruption. The most fundamental topological features of the shape present in the image are the number of connected components, number of holes and its higher dimensional counterparts \citep{zhou2017exploring}. The topological primitives we target are holes and relative geometry, which we combine with conventional shape channels, and classify with lightweight prototype learning. 

The goal in creating such an explicit parametrization of the shape topology is to provide invariance under specific affine transformations such as rotation, translation and scaling (RTS). Topology of a shape (for e.g. number of holes, connectivity) is invariant under any continuous deformation, known as homeomorphism. Any manner of RTS-transformation is continuous and preserve hole count, containment, and adjacency. So using holes as discrete primitives makes it easy to build RTS-invariant features: count holes, encode hole-vs-outer relations, and canonicalize coordinates.
We provide theoretical justification that (i) our primitives are RTS-invariant by construction and (ii) HD bundling/prototypes are stable under token noise and bit flips, which provides robustness observed under corruptions.  

We would like to note that the idea that machine learning models explicitly or implicitly learn the topological structure of the data has been studied in some detail. For example,  \cite{carlsson2020topological} showed that adding static topology-inspired image filters (based on a second-embedded Klein bottle) can improve robustness under heavy corruption, while \cite{naitzat2020topology} found that DNN feature representations become topologically simpler across layers, evidenced by reductions in Betti numbers. Building on these insights, we explicitly construct topological features within HD computing to improve robustness.

\section{Topological signature encoding}
We cast the problem of image classification as a shape identification problem - \textit{given different observations of an object under nuisance transformation and noise, our HD classifier has to decide whether they represent the same underlying shape.} 

For a grayscale image $x: \Omega \rightarrow \mathbb{R}$, the true shape is a compact subset $\mathcal{A} \subset \mathbb{R}^2$ of the foreground region. An observed image produces an estimate $\mathcal{\hat{A}}$ via thresholding/cleanup. Shape identification typically considers the equivalence class 

\begin{equation}
    [\mathcal{A}] = {g(\mathcal{A}): g \in G}
\end{equation}

where $G$ is the nuisance group. Our goal is to construct an ideal map 

\begin{equation}
    \Psi (\mathcal{A}) \in \{-1, +1\}^D
\end{equation}
such that $\Psi(g(\hat{\mathcal{A}})) \approx \Psi(\hat{\mathcal{A}})$ for all $g \in G$ and $\Psi$ is discriminative across different classes of shapes. 

\begin{wrapfigure}
    {r}{0.25\textwidth}
  \begin{center}
\includegraphics[width=0.25\textwidth]{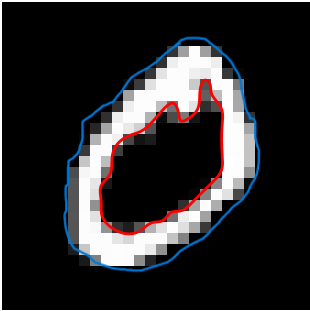}
  \end{center}
  \caption{Outer contour (blue) and hole contour (red).}
  \label{fig:outer_hole_contour}
\end{wrapfigure}

The shape $\mathcal{A}$ in our case could be framed in two parts - global geometry through the shape of outer boundary $\mathcal{A}_O$ and geometry of the holes (if any) present inside the shape, $\mathcal{A}_{\mathcal{H}_i}$, where $\mathcal{H}_i$ represents the $i^{th}$ hole, $0 \leq \mathcal{H}_i \leq \mathcal{H}_{max}$. The outer contour is the usual object silhouette information which represents the boundary/appearance-based invariants commonly used in classical shape matching. 

Holes, on the other hand, have been utilized as primary topological features (Betti-1 in the binary region) due to their invariance under any affine transformation \citep{edelsbrunner2008persistent}. However, the location and shape of the individual hole is dependent on the outer contour, i.e. the scale of the overall shape dictates the relative location of the hole. Therefore, the first step is to find the existence of holes via connected component analysis. Given that there are holes present, we need to fix the relative location of the individual hole using an RTS-canonical frame of the outer contour and finally, provide the intrinsic shape of the hole using rotation/scale-invariant Fourier magnitudes of the radial signature (Fig. \ref{fig:outer_hole_contour} shows the outer and hole contours). The methodology is described in detail below:

\subsection{Discrete Topological Primitives + Set Encoding}
We focus on \emph{discrete topological primitives}-connected components, holes, and contours-rather than persistent homology. For an image $x$, we extract a multiset of primitives
\[
\mathcal{P}(x)=\{(t_i, \mathcal{G}_i, w_i)\}_{i=1}^{m(x)},
\]
where $t_i$ is a \emph{type tag} (e.g., outer contour, hole contour, connected component), $\mathcal{G}_i$ is the primitive geometry (e.g., a sampled contour/spline or a component mask) and $w_i$ is the reliability weight associated with the type tag. This yields a representation analogous to ``bag-of-parts” style pooling, where repeated or consistent primitives reinforce, thus building a global signature that is stable to permutations and moderate changes in the primitive set. Compared to an explicit correspondence problem, where each primitive has to align, the set representation accommodates \textit{variable cardinality}: shapes having different number of holes have similar representation dimension; and is robust to \textit{set corruption}: difference in shape of a hole does not lead to drastically different results. 

Each primitive is mapped to a real-valued feature vector via a canonicalized descriptor
\[
\Psi(\mathcal{G}_i; x)\in\mathbb{R}^k,
\]
designed to be invariant to translation/rotation/scale (RTS) by construction. We then encode each primitive into a hypervector
\[
h_i \;=\; \mathbf{r}(t_i)\odot \Phi(\Psi(\mathcal{G}_i;x))\in\{-1,+1\}^D,
\]
where $\mathbf{r}(t_i)\in\{-1,+1\}^D$ is a role vector determined by the primitive type and $\Phi:\mathbb{R}^k\to\{-1,+1\}^D$ is a continuous-to-hypervector encoder (e.g., normalized random projections / random hyperplane sign). The final representation is the \emph{set hypervector}
\[
\mathbf{HV}(x)\;=\;\mathrm{sign}\!\left(\sum_{i=1}^{m(x)} w_i\, h_i\right),
\]
which is permutation-invariant in the ordering of extracted primitives and naturally supports variable-sized sets. We now describe the different primitive construction in some detail. 

\subsection{Outer shape extraction and moment descriptors}
For the grayscale image $x$, we isolate the foreground pixels using a global or adaptive threshold map $T : \Omega \rightarrow \mathbb{R}$ to form a binary mask $b(p),    p \in \Omega$, whereby $\mathcal{C}(b)$ provide the set of connected components of $\{p \in \Omega: b(p) = 1\}$. The shape is then given by the largest connected component
\begin{equation}
    \mathcal{A} = \argmax_{\mathcal{C} \in \mathcal{C}(b)} |C|,     b_\mathcal{A} (p) = \mathbbm{1} \{p \in \mathcal{A}\}.
\end{equation}
 A continuous region is then obtained via pixel embedding, 

\begin{equation}
    \tilde{Re} = \bigcup_{(i,j) \in \mathcal{A}} \Bigg( \Big[i - \dfrac{1}{2}, i + \dfrac{1}{2}\Big] \times \Big[j - \dfrac{1}{2}, j + \dfrac{1}{2}\Big] \Bigg) \subset \mathbb{R}^2.
\end{equation}

The outer contour is then defined as the outer boundary component of $\partial \tilde{Re}$. Let $\partial\Omega$ denote the set of border pixels. Let $\mathcal{C}(\bar b)$ be the set of connected components of $\{p:\bar b(p)=1\}$, and define
\begin{equation}
\mathcal{H}(\mathcal{A})
=
\Big\{
\mathcal{H}\in\mathcal{C}(\bar b)\;:\; \mathcal{H}\cap \partial\Omega=\emptyset
\Big\}.
\end{equation}
Each $\mathcal{H}\in\mathcal{H}(\mathcal{A})$ is a hole region. For each hole region $\mathcal{H}$, we extract its boundary (contour) as an isocontour at level $1/2$ of the indicator image of $\mathcal{H}$ (marching squares), yielding an ordered set of points
\begin{equation}
C_\mathcal{H}=\{c_1,\dots,c_{n_\mathcal{H}}\}\subset\mathbb{R}^2,
\end{equation}
which we treat as a closed polygonal curve (with $c_{n_\mathcal{H}+1}=c_1$).

Let $C_O=\{o_1,\dots,o_{n_0}\}\subset\mathbb{R}^2$ denote the outer contour of the foreground (e.g., the longest extracted contour). We define a canonical RTS frame using the outer contour. First the outer centroid $c =\frac{1}{n_0}\sum_{i=1}^{n_0} o_i \in\mathbb{R}^2,$ and a scale $s = \sqrt{\frac{1}{n_0}\sum_{i=1}^{n_0}\|o_i-c\|_2^2} + \varepsilon,$ is computed, where $\varepsilon>0$ prevents division by zero. To fix rotation, we compute the principal axis of the centered contour points. Let $O=[o_1-c;\dots;o_{n_0}-c]\in\mathbb{R}^{n_0\times 2}$ and define the covariance
\begin{equation}
\Sigma=\frac{1}{n_0}O^\top O \in\mathbb{R}^{2\times 2}.
\end{equation}
Let $v\in\mathbb{R}^2$ be the unit eigenvector associated with the largest eigenvalue of $\Sigma$, and let $\theta=\mathrm{atan2}(v_2,v_1)$. The canonical rotation is
\begin{equation}
R
=
\begin{bmatrix}
\cos(-\theta) & -\sin(-\theta)\\
\sin(-\theta) & \cos(-\theta)
\end{bmatrix}.
\end{equation}
Since PCA orientation is ambiguous up to sign ($v$ and $-v$), we apply a deterministic sign-fix (e.g., choose the sign such that the skewness of projections onto $v$ is positive), yielding a unique $R$ in practice.

Given any point $p\in\mathbb{R}^2$ (e.g., a hole centroid), its RTS-canonical coordinates are
\begin{equation}
\mathrm{canon}(p)
=
R\,\frac{(p-c)}{s}\in\mathbb{R}^2.
\end{equation}

In order to preserve invariance under RTS, we then use Zernike moments as global descriptors of the area within $C_O$ \citep{novotni2004shape, li2008complex} defined as:
\begin{equation}
    Z_{n,m}(f) = \dfrac{n+1}{\pi} \int_{\rho \leq 1} f(\rho, \theta) \overline{V_{n,m}(\rho, \theta)} dA
\end{equation}
where Zernike moment is of order $(n, m)$ with $f(\rho, \theta)$ being the normalized intensity function over the unit disk obtained by warping the region inside $C_O$ into polar coordinates and $\overline{V_{n,m}(\rho, \theta)}$ being the orthogonal radial polynomials. Zernike moments (especially when using magnitudes for rotation invariance) provide a compact summary of the overall mass distribution within the outer shape, but they can be relatively insensitive to fine boundary irregularities, thin strokes, and small localized differences that distinguish visually similar classes.

\paragraph{Spatial pyramid Zernike descriptor} We implement a spatial pyramid version of the Zernike descriptor retaining Zernike moments' strong global invariance properties while recovering spatial layout information by computing moments not only on the full image but also on a coarse grid of subregions. Let $I \in [0, 1]^{\text{height} \times \text{width}}$ denote the normalized grayscale image (after centering/scaling), and let $\mathbf{g} = \{ B_i\}_{i=1}^g$ be a rectangular grid partition of $[1,H] \times [1,W]$ into $g = g_y g_x$ cells. For any region $Re$, let $x_{Re}(I)$ denote the corresponding image patch, embedded into a square array by zero-padding and then resampled to a fixed resolution $P\times P$ (so that Zernike polynomials are evaluated on a consistent discrete unit-disk support). Writing $Z_{n,m}(\cdot;r)\in \mathbb{C}$ for the Zernike moment of order $(n,m)$ computed on a disk of radius $r$ (with $n\leq d$), we define the level-0 (global) feature vector as $\phi_0(I)=(|Z_{n,m}(I;r0)|)_{(n,m) \in \mathcal{\gimel}_d}$ and the level-1 (local) feature vectors as $\phi_i(I)=(|Z_{n,m}(\Pi_{B_i}(I);r1)|)_{(n,m) \in \mathcal{\gimel}_d}$, where $\mathcal{\gimel}_d=\{(n,m):0 \leq n \leq d, |m| \leq n,n-|m| \text{ even}\}$. The final descriptor is the concatenation
\begin{equation}
    \phi_{SPZ}(I)=[\phi_0(I); \phi_1(I); \cdots; \phi_\mathbf{g} (I)] \in \mathbb{R}^{(\mathbf{g}+1)|\mathcal{\gimel_d}|}.
\end{equation}
Using magnitudes $|Z_{n,m}|$ yields rotation invariance within each region (since rotations induce only a phase shift $e^{-im\alpha}$ in $Z_{n,m}$), while the grid decomposition preserves coarse spatial layout by preventing cancellations across distant parts of the shape.

In order to keep the entire spatial pyramid rotation invariant, we estimate the canonical orientation $R$ from the outer contour. We then rotate the normalized patch into that frame before computing grid features: 
\begin{equation}
    I_{can}(p) = I(R^{-1} p).
\end{equation}
Then the grid is effectively attached to the actual object, and concatenating per-cell features becomes approximately rotation invariant. 


\paragraph{HOG descriptor} While Zernike provides the global structure, it lacks the local structure which introduces subtle variations that differentiates images. Thus, we provide the histogram of oriented gradients (HOG) $\phi_{HOG}$ features of the images to capture local, orientation-driven stroke structure that is often attenuated by global moment descriptors. HOG represents an image by histograms of local gradient orientations pooled over spatial cells with block normalization, making it well suited to character-like data where discrimination depends on the presence, direction, and relative placement of edges. 

HOG is naturally robust to moderate illumination/contrast changes and small geometric perturbations due to its local pooling and normalization. Using both descriptors therefore yields a more informative outer-channel representation: Zernike captures global shape geometry in a low-dimensional, rotation-stable form, while HOG supplies high-frequency, local edge cues, improving separability under noise and partial corruption before the features are mapped into the HD space.

\subsection{Hole extraction and Fourier descriptor}
\label{subsec:holes}


\paragraph{Hole geometry features (RTS-invariant relative position)}
For each hole region $\mathcal{H}$, define its centroid
\begin{equation}
\mu_{\mathcal{H}}=\frac{1}{|\mathcal{H}|}\sum_{p\in \mathcal{H}} p \in\mathbb{R}^2,
\end{equation}
and its canonical relative position
\begin{equation}
(\tilde y_\mathcal{H},\tilde x_\mathcal{H})=\mathrm{canon}(\mu_\mathcal{H}).
\end{equation}
By construction, $(\tilde y_\mathcal{H},\tilde x_\mathcal{H})$ is invariant to global translation and isotropic scale, and is approximately invariant to global rotation (up to the stability of the outer-frame estimate).

\paragraph{Invariant hole shape descriptor}
We compute an intrinsic, rotation-invariant descriptor for each hole contour $C_\mathcal{H}$. First, we resample $C_\mathcal{H}$ to a fixed number of points and fit a periodic spline to obtain a smooth, closed parameterization $\gamma_\mathcal{H}:[0,2\pi)\to\mathbb{R}^2$. Let $\{c_j=\gamma_\mathcal{H}(\theta_j)\}_{j=1}^{N}$ be uniform samples with $\theta_j=\frac{2\pi(j-1)}{N}$, and let $\bar c_\mathcal{H}=\frac{1}{N}\sum_{j=1}^{N} c_j$ be the contour centroid. Define the radial signature
\begin{equation}
r_\mathcal{H}(\theta_j)=\|c_j-\bar c_\mathcal{H}\|_2,\qquad j=1,\dots,N.
\label{eq:radial}
\end{equation}
 We then take the first $k$ non-DC Fourier magnitudes (excludes the average, zero-frequency offset) as the hole
shape feature:
\begin{equation}
\phi_{\mathrm{shape}}(\mathcal{H})
=
\big(|\widehat r_\mathcal{H}[1]|,\dots,|\widehat r_\mathcal{H}[k]|\big)\in\mathbb{R}^{k}.
\end{equation}

\paragraph{Per-hole feature vector}
Finally, we concatenate the intrinsic hole shape and canonical relative geometry into a fixed-length per-hole descriptor,
\begin{equation}
f(\mathcal{H})
=
\big[\phi_{\mathrm{shape}}(\mathcal{H});\ \tilde y_\mathcal{H};\ \tilde x_\mathcal{H}\big]\in\mathbb{R}^{k+2}.
\end{equation}

\subsection{Determining the reliability weight of the type tag}
While our working hypothesis is that the topological shape cues—captured by Zernike moments and hole-based descriptors—provide complementary information to HOG (particularly under corruptions such as noise, blur, occlusion, or small geometric distortions), the relative reliability of these feature families is not known a priori and can vary by dataset and corruption type. To avoid degrading performance by over-emphasizing an unreliable cue, we treat the contribution of each descriptor as a weighted component and learn these weights on a held-out validation set.

Let $x$ denote an input image and let 
$c \in \{1, \dots, C\}$ index the class. We compute three hypervector (HV) representations:
\begin{itemize}
    \item $\mathbf{HV}^{hog} \in \{-1, +1\}^D$ from HOG features, 
    \item $\mathbf{HV}^{zer} \in \{-1, +1\}^D$ from Zernike moments, 
    \item $\mathbf{HV}^{hole} \in \{-1, +1\}^D$ from hole features.
\end{itemize}

For each descriptor type $t \in \{\text{hog, zernike, hole}\}$, we maintain class prototypes (accumulators) $\mathbf{P}^t_c \in \mathbb{Z}^D$, learned from the training set (e.g. by bundling and online updates). 

We score each class using cosine similarity between the sample HV and the prototype:

\begin{equation}
    s^t_c (x) = cos(\mathbf{HV}^t(x), \mathbf{P}^t_c) = \dfrac{\langle \mathbf{HV}^t(x), \mathbf{P}^t_c \rangle}{||\mathbf{HV}^t(x)||_2 ||\mathbf{P}^t_c||_2 + \epsilon}
\end{equation}
where $\epsilon$ provides numerical stability. We then combine these scores with non-negative reliability weights $\alpha$ and $\beta$:

\begin{equation}
    S_c(x; \alpha, \beta) = S^{hog}_c(x) + \alpha S^{zer}_c (x) + \beta S^{hole}_c (x).
\end{equation}
The predicted class is given by $\hat{y}(x; \alpha, \beta) = \argmax_c Sc(x; \alpha, \beta)$.

\subsection{Theoretical underpinnings}

\begin{theorem}
    The spatial pyramid outer shape descriptor is invariant to rotation, translation and scale. 
\end{theorem}

\begin{proof}
Assume that under any similarity transform $tr(p)=aQp+t$ with $a>0$, $Q\in SO(2)$, and $t\in\mathbb{R}^2$,
the pyramid regions transform equivariantly (i.e., $tr({Re}_i)={Re}_i$ after normalization/resampling), and that
$\mathcal{N}_{Re}$ is computed from $f$ on $Re$ so that $c_{tr(Re)}=tr(c_{Re})$ and $s_{tr(Re)}=a\,s_{Re}$.
Then for the transformed image $f^tr(p)=f(tr^{-1}p)$,

Translation invariance follows because centering by $c_{Re}$ satisfies
$c_{g(Re)}=aQ c_{Re}+t$, hence $(p-c_{g(Re)})/s_{g(Re)} = Q(p'-c_{Re})/s_{Re}$ for $p'=g^{-1}p$.
Scale invariance follows from $s_{g(Re)}=a s_{Re}$, so normalization cancels the factor $a$.
After centering and scaling, a rotation $Q$ acts as a rotation on the unit disk coordinates:
$u\mapsto Qu$. Zernike moments satisfy the standard rotation law

where $Q$ is rotation by angle $\alpha$, hence $|Z_{n,m}|$ is rotation invariant.
Applying the same argument region-wise and concatenating over the pyramid yields the result.

\end{proof}

\begin{theorem}
    The hole shape descriptor is invariant to rotation, translation and scale. 
\end{theorem}

\begin{proof}
For a radial signature defined in Eq. \ref{eq:radial}, a rotation of the hole by angle $\alpha$ induces a circular shift
$r_\mathcal{H}(\theta)\mapsto r_\mathcal{H}(\theta-\alpha)$. Therefore, the magnitudes of the discrete Fourier transform
(DFT) of $r_\mathcal{H}$ are rotation invariant. Let
\begin{equation}
\widehat r_\mathcal{H}[m] = \sum_{j=1}^{N} r_\mathcal{H}(\theta_j)\,e^{-i2\pi (j-1)m/N},\qquad m=0,\dots,N-1,
\end{equation}
and normalize by the RMS energy of $r_\mathcal{H}$,
\begin{equation}
r_\mathcal{H} \leftarrow \frac{r_\mathcal{H}}{\sqrt{\frac{1}{N}\sum_{j=1}^{N} r_\mathcal{H}(\theta_j)^2}+\varepsilon},
\end{equation}
which yields scale invariance. Also regarding start-point and orientation ambiguities, any reparameterization $\theta \mapsto \theta - \theta_0$ (start point change) yields a circular shift of $r_\mathcal{H}$, and changing contour orientation $\theta \mapsto - \theta$ yields $\widehat r_\mathcal{H}[m] \mapsto \overline{\widehat r_\mathcal{H}[m]}$. Thus, the hole shape descriptor is invariant to both start-point and orientation ambiguities.
\end{proof}


\section{Results}
We evaluate the proposed topological encoding method on two standard vision benchmarks spanning increasing visual complexity: MNIST \citep{lecun2002gradient} and EMNIST \citep{cohen2017emnist}; on an Ubuntu GPU server machine with two NVIDIA A6000 cards and Intel® Xeon(R) Bronze 3206R CPU @ 1.90GHz × 16. MNIST provides a controlled setting of centered grayscale digits, EMNIST extends this regime to a larger and more diverse set of handwritten characters. Specifically, we utilize the letters of the EMNIST dataset, which have small strokes that vanish at low resolution. Together, these datasets allow us to assess both in-distribution accuracy and robustness as the input distribution shifts from clean binary-like shapes to natural scenes. The algorithm pseudocode is provided in Appendix \ref{app_sec:algo} and the hyperparameter settings are provided in Appendix \ref{app_sec:hyper}.

We further stress-test generalization by evaluating under common distribution shifts induced by synthetic corruptions applied to the test sets. This provides a realistic measure of generalization of the method to unseen variations. Specifically, we generate corrupted variants of each test dataset using: (i) in-plane rotations up to a specified maximum angle, (ii) additive Gaussian noise with controlled standard deviation, (iii) salt-and-pepper impulse noise at a fixed flip probability, (iv) cut-out occlusions that mask a contiguous square region of fixed size, and (v) zoom transformations that rescale the foreground content and re-center it on the original canvas. These perturbations target complementary failure modes—geometric misalignment, stochastic pixel noise, sparse outliers, partial occlusion, and scale changes—allowing us to quantify how well the proposed encoding preserves class separability when the input deviates from the training distribution.

\subsection{MNIST dataset}

\begin{table}[htbp!]
\centering
\caption{Test accuracy for Naive HDC and Topology-guided HDC for MNIST dataset before and after online training for 20 epochs under clean and corrupted settings.}
\label{tab:mnist_results_naive_vs}
\begin{tabular}{lcc|cc}
\toprule
& \multicolumn{2}{c|}{\textbf{Naive HDC}} & \multicolumn{2}{c}{\textbf{Topology-guided HDC}} \\
\textbf{Setting} & \textbf{Before train} & \textbf{After train} & \textbf{Before train} & \textbf{After train} \\
\midrule
Baseline (clean)                 & 82.10 & 95.50 & 93.96 & 97.68 \\
Rotation ($20^\circ$)            & 20.50 & 18.30 & 93.89 & 97.61 \\
Gaussian noise ($\sigma=0.1$)    & 6.86 & 9.78 & 82.99 & 88.79 \\
Gaussian noise ($\sigma=0.2$)    & 7.53 & 6.91 & 71.39 & 76.93 \\
Salt-and-pepper ($p=0.1$)        & 43.86 & 65.84 & 72.88    & 77.29     \\
Cut out ($size=4$)               & 62.19 & 44.47 & 93.22 & 97.03 \\
Zoom (scale $=0.75$)             & 19.86 & 14.89 & 87.41 & 92.29    \\
\bottomrule
\end{tabular}
\end{table}

Table \ref{tab:mnist_results_naive_vs} compares a baseline Naive HDC classifier to our Topology-guided HDC under test-time corruptions. Naive HDC, which relies on appearance features, degrades sharply even under mild shifts. Topology-guided HDC augments appearance with geometric/topological descriptors (e.g., global shape and hole structure) and is consistently more robust, both before online adaptation (clean bundled prototypes) and after 20 epochs of OnlineHD-style updates. Overall, these results indicate that explicit topological structure yields stable performance across multiple corruption families.

Table \ref{tab:mnist_results_cnn_vs} further compares Topology-guided HDC to a compact CNN for 28×28 grayscale digit classification. The CNN includes five convolutional layers with batch normalization, pooling, and dropout, followed by a small fully connected head. We report accuracy after training on clean MNIST and testing under corruptions, without corruption-specific augmentation unless noted.
\begin{table}[htbp!]
\centering
\caption{Test accuracy for Topology-guided HDC and CNN for MNIST dataset before and after online training for 20 epochs under clean and corrupted settings.}
\label{tab:mnist_results_cnn_vs}
\begin{tabular}{lcc|c}
\toprule
& \multicolumn{2}{c|}{\textbf{Topology-guided HDC}} & \multicolumn{1}{c}{\textbf{CNN}} \\
\textbf{Setting} & \textbf{Before train} & \textbf{After train} & \textbf{After train} \\
\midrule
Baseline (clean)                 & 93.96 & 97.68 & \textbf{99.1}\\
Rotation ($20^\circ$)           & 93.89 & 97.61 & \textbf{99.16} \\
Gaussian noise ($\sigma=0.1$)    & 82.99 & \textbf{88.79} & 11.35 \\
Gaussian noise ($\sigma=0.2$)   & 71.39 & \textbf{76.93} & 11.01 \\
Salt-and-pepper ($p=0.1$)        & 72.88    & \textbf{77.29} & 40.1   \\
Cut out ($size=4$)               & 93.22 & 97.03 & \textbf{98.52} \\
Cut out ($size=8$)               & 89.72 & \textbf{93.34} & 92.2 \\
Cut out ($size=12$)               & 83.01 & \textbf{86.27} & 81.87 \\
Zoom (scale $=0.75$)             & 87.41 & 92.29 & \textbf{94.36}  \\
Zoom (scale $=0.5$)              & 44.02 & \textbf{58.23} & 52.85 \\
\bottomrule
\end{tabular}
\end{table}

On the clean baseline, the CNN achieves the highest accuracy (99.1\%), as expected given its strong capacity to fit the MNIST distribution. Under additive noise, the CNN collapses to near chance ($\approxeq 11.01\%$ for $\sigma \in \{0.1,0.2\}$), while Topology-guided HDC remains much more robust (88.79\% and 76.93\%) and also outperforms under salt-and-pepper (77.29\% vs. 40.1\%). For cutout, the CNN is slightly better at small occlusions, but Topology-guided HDC degrades more gracefully as occlusion increases (93.34\% at size 8; 86.27\% at size 12). Under zoom, the CNN leads at mild scaling, whereas Topology-guided HDC is more resilient at stronger scaling (58.23\% vs. 52.85\% at scale 0.5).
Taken together, Tables \ref{tab:mnist_results_naive_vs} and \ref{tab:mnist_results_cnn_vs} demonstrate two key findings. First, incorporating topology-guided shape information into an HDC framework yields large robustness gains over a naive HDC baseline across diverse corruptions, even before online adaptation and consistently after online training. Second, while the CNN achieves superior accuracy on clean MNIST, Topology-guided HDC can offer markedly better robustness for several common corruption types—especially those that disrupt local pixel statistics—without requiring corruption-specific training. These results underscore the value of topology-guided representations as a complementary route to robustness, particularly in settings where retraining or extensive augmentation is undesirable.

\subsection{EMNIST dataset}
\begin{table}[htbp!]
\centering
\caption{Test accuracy for Topology-guided HDC and CNN for EMNIST letters dataset before and after online training for 20 epochs under clean and corrupted settings.}
\label{tab:emnist_results_cnn_vs}
\begin{tabular}{lcc|c}
\toprule
& \multicolumn{2}{c|}{\textbf{Topology-guided HDC}} & \multicolumn{1}{c}{\textbf{CNN}} \\
\textbf{Setting} & \textbf{Before train} & \textbf{After train} & \textbf{After train} \\
\midrule
Baseline (clean)                 & 75.38 & 89.37 & \textbf{93.21}\\
Rotation ($20^\circ$)           & 75.55 & 89.25 & \textbf{93.21} \\
Gaussian noise ($\sigma=0.1$)    & \textbf{57.7} & 50.80 & 3.84\\
Salt-and-pepper ($p=0.1$)        & \textbf{42.01} & 40.06 & 31.51 \\ 
Cut out ($size=4$)               & 74.22 & 87.22 & \textbf{91.01} \\
Cut out ($size=8$)               & 70.06 & 82.11 & \textbf{82.34} \\
Cut out ($size=12$)               & 61.44 & \textbf{69.97} & 67.98 \\
Zoom (scale $=0.75$)             & 4.51 & \textbf{5.12} & 3.85 \\
Zoom (scale $=0.5$)              & 4.31 & \textbf{5.08} & 3.84 \\
\bottomrule
\end{tabular}
\end{table}

For the letters of the EMNIST dataset, the Naive HDC performs worse than the MNIST dataset and therefore is not shown. We provide a comparison with the CNN in Table \ref{tab:emnist_results_cnn_vs}. The Table shows that while the CNN achieves the highest accuracy on clean EMNIST letters (93.21\%), Topology-guided HDC provides markedly better robustness to pixel-noise corruptions. In particular, under Gaussian noise ($\sigma=0.1$) and salt-and-pepper noise ($p=0.1$), Topology-guided HDC substantially outperforms the CNN before online training (57.7\% vs. 3.84\% and 42.01\% vs. 31.51\%), indicating that the structural/topological channel remains informative when local appearance cues are degraded. Online training improves clean and cutout performance for Topology-guided HDC (e.g., $75.38\%\rightarrow89.37\%$ on clean; $61.44\%\rightarrow69.97\%$ at cutout size 12), but can reduce robustness to Gaussian noise, suggesting a tradeoff where adaptation to the clean distribution may shift prototypes away from corruption-invariant structure.

\section{Discussion}
This work proposes an explicit topological encoding pipeline for HDC that decomposes a binarized shape into a variable-size set of primitives (outer contour, holes, and their relative geometry), maps each primitive to a hypervector, and aggregates them into a single representation with learned reliability weights across feature types. Across MNIST/EMNIST corruption benchmarks, the resulting topology-guided HDC improves robustness relative to a naive pixel-based HDC baseline, and in several regimes approaches the clean accuracy of a compact CNN while maintaining substantially better stability under specific distribution shifts.

A consistent pattern is that the proposed representation is most beneficial under corruptions that disrupt local pixel statistics but preserve coarse structure: moderate geometric misalignment, partial occlusion, and certain contrast/appearance changes. This aligns with the design of the feature channels:
(i) Outer-shape descriptors (Zernike + spatial pyramid) emphasize global mass distribution and coarse spatial layout after RTS canonicalization, making them less sensitive to small local perturbations than raw pixel binding; (ii) Hole features capture class-relevant topology (e.g., number and relative placement of holes) that is largely invariant under continuous deformations and provides a complementary cue when stroke-level appearance is degraded.

\subsection{Limitations}
However, the results also highlight failure modes. Our approach depends on reliable binarization and primitive extraction (connected components, contours, and holes); under severe noise or low contrast these steps can become unstable (e.g., fragmented boundaries or spurious holes), which can dominate downstream accuracy. Theoretical guarantees cover similarity transforms (rotation/translation/scale) given correct canonicalization, but do not address more complex nuisances such as nonrigid deformations, perspective effects, blur/compression, or domain shifts with strong background clutter. 
Performance may also be sensitive to design choices (thresholding/cleanup, Zernike order, spatial pyramid resolution, number of Fourier coefficients, and reliability-weight calibration), and we do not exhaustively characterize this sensitivity in the main paper. Finally, while HDC inference is lightweight, preprocessing can be the computational bottleneck depending on implementation and image resolution. This suggests that the method’s robustness is limited by the stability of the segmentation stage and by whether the ``topological signal” remains identifiable after corruption.

{
\small
\bibliographystyle{apalike}
\bibliography{bibtex}
}

\newpage
\appendix

\section{Algorithm}
\label{app_sec:algo}

\input{algo}

\newpage
\section{Hyperparameter settings}
\label{app_sec:hyper}
\begin{table}[hbtp!]
\centering
\small
\caption{Hyperparameter settings used in experiments.}
\label{tab:hyperparams}
\begin{tabular}{ll}
\toprule
\textbf{Category} & \textbf{Setting} \\
\midrule
HD representation & Dimension $D=10{,}000$; \\
& bipolar hypervectors in $\{-1,+1\}^D$ \\
Classes & $C=10$ (MNIST) and $C=26$ (EMNIST)\\
Batching & Batch size $B=512$ \\
Online adaptation & OnlineHD-style updates for $E=20$ epochs with step size $\eta=1$ \\
Fusion weights &  $\alpha,\beta$ tuned on validation grid $\{0,0.1,\dots,1.0\}$ \\
Roles (seeds) & $\texttt{ROLE\_OUTER}$ seed $999$; \\
& $\texttt{ROLE\_HOG}$ seed $1000$;\\ 
& $\texttt{ROLE\_HOLES}$ seed $1001$ \\
Random projection & HOG encoder seed $2$;\\
& Zernike encoder seed $1$ \\
Hole-set encoder & Quantization levels $Q=101$;\\ 
& hole feature length $k = k_{\mathrm{shape}} + 4$ with $k_{\mathrm{shape}}=12$ (thus $k=16$);\\
& seed $123$ \\
\midrule
Corruptions (test-time) & Rotation $20^\circ$; \\
& Gaussian noise $\sigma\in\{0.1,0.2\}$; \\
& salt-and-pepper $p=0.1$; \\
& cutout size $\in\{4,8,12\}$; \\
& zoom scale $\in\{0.75,0.5\}$ \\
\bottomrule
\end{tabular}
\end{table}





\end{document}

%% file: algo.tex
\begin{algorithm}[hbtp!]
\DontPrintSemicolon
\SetKwComment{tcp}{$\vartriangleright$}{}
\SetKwInput{KwIn}{Input}
\SetKwInput{KwOut}{Output}

\KwIn{Dataset $(X_{\mathrm{tr}},y_{\mathrm{tr}}),(X_{\mathrm{te}},y_{\mathrm{te}})$; validation fraction $\rho$; corruption spec $\mathcal{C}$; HV dimension $D$; classes $C$; batch size $B$; OnlineHD epochs $E$ and step size $\eta$; hole params $(k_{\mathrm{shape}},k,Q,H_{\max})$; fusion grids $\mathcal{A},\mathcal{B}$.}
\KwOut{Predictions $\hat y_{\mathrm{te}}^{(0)},\hat y_{\mathrm{te}}^{(1)}$; prototypes before/after $(P_h,P_z,P_o)$ and $(P_h',P_z',P_o')$; fusion weights $(\alpha,\beta),(\alpha',\beta')$.}

\BlankLine
$(X_{\mathrm{tr}},y_{\mathrm{tr}},X_{\mathrm{val}},y_{\mathrm{val}})\gets \mathrm{StratifiedSplit}(X_{\mathrm{tr}},y_{\mathrm{tr}};\rho)$\;
$\widetilde X_{\mathrm{te}}\gets \mathrm{ApplyCorruption}(X_{\mathrm{te}};\mathcal{C})$\;

\BlankLine
\ForEach{$S\in\{\mathrm{tr},\mathrm{val},\mathrm{te}\}$}{
  \ForAll{$x \in X_S$}{
    \tcp*[r]{parallel over $x$}
    $x'\gets \mathrm{NormalizeGlyph}(x)$\;
    $f^z(x)\gets \mathrm{SpatialPyramidZernike}(x')$\;
    $f^h(x)\gets \mathrm{HOG}(x')$\;
    $\mathcal{F}^{\mathcal{H}}(x)\gets \mathrm{HoleFeatures}(x';k_{\mathrm{shape}})$\;
  }
  $(HF_S,M_S)\gets \mathrm{PadOrTruncateHoles}(\{\mathcal{F}^{\mathcal{H}}(x)\}_{x\in X_S};H_{\max},k)$\;
}
$HFEATS_{\mathrm{tr}}\gets \mathrm{Flatten}(\{\mathcal{F}^{\mathcal{H}}(x)\}_{x\in X_{\mathrm{tr}}})$\;

\BlankLine
$\mathrm{Enc}_z\gets \mathrm{FitRandomProj}(\{f^z(x)\}_{x\in X_{\mathrm{tr}}};D)$\;
$\mathrm{Enc}_h\gets \mathrm{FitRandomProj}(\{f^h(x)\}_{x\in X_{\mathrm{tr}}};D)$\;
$\mathrm{Enc}_o\gets \mathrm{FitHoleSetHDC}(HFEATS_{\mathrm{tr}};D,Q,k)$\;
$\mathbf r_z,\mathbf r_h,\mathbf r_o \sim \{-1,+1\}^D$ \tcp*[r]{role hypervectors}

\BlankLine
$P_h,P_z,P_o \gets 0_{C\times D}$\;
\ForEach{mini-batch $I\subseteq\{1,\dots,|X_{\mathrm{tr}}|\}$, $|I|\le B$}{
  $H_h\gets \mathrm{Bind}(\mathrm{Enc}_h(f^h[I]),\mathbf r_h)$\;
  $H_z\gets \mathrm{Bind}(\mathrm{Enc}_z(f^z[I]),\mathbf r_z)$\;
  $H_o\gets \mathrm{Bind}(\mathrm{Enc}_o(HF_{\mathrm{tr}}[I],M_{\mathrm{tr}}[I]),\mathbf r_o)$\;
  \tcp{Accumulate class prototypes (by summation)}
  \ForEach{$i\in I$}{
    $c \gets y_{\mathrm{tr}}[i]$\;
    $P_h[c]\gets P_h[c] + H_h[i]$\;
    $P_z[c]\gets P_z[c] + H_z[i]$\;
    $P_o[c]\gets P_o[c] + H_o[i]$\;
  }
}

\BlankLine
$(\alpha,\beta)\gets \arg\max\limits_{a\in\mathcal{A},\,b\in\mathcal{B}}
\mathrm{Acc}\!\left(y_{\mathrm{val}},\mathrm{Predict}(P_h,P_z,P_o,a,b;\mathrm{val})\right)$\;
$\hat y_{\mathrm{te}}^{(0)} \gets \mathrm{Predict}(P_h,P_z,P_o,\alpha,\beta;\mathrm{te})$\;

\BlankLine
\For(\tcp*[r]{OnlineHD updates, per channel}){$e\gets 1$ \KwTo $E$}{
  \ForEach{mini-batch $I$ of size $\le B$ (shuffled)}{
    $H_h\gets \mathrm{Bind}(\mathrm{Enc}_h(f^h[I]),\mathbf r_h)$\;
    $H_z\gets \mathrm{Bind}(\mathrm{Enc}_z(f^z[I]),\mathbf r_z)$\;
    $H_o\gets \mathrm{Bind}(\mathrm{Enc}_o(HF_{\mathrm{tr}}[I],M_{\mathrm{tr}}[I]),\mathbf r_o)$\;

    \ForEach{$i\in I$}{
      \ForEach{$t \in \{h,z,o\}$}{
        $\hat c_t \gets \arg\max_{c\in\{1,\dots,C\}} \cos(H_t[i], P_t[c])$\;
        \If{$\hat c_t \ne y_{\mathrm{tr}}[i]$}{
          $P_t[y_{\mathrm{tr}}[i]] \gets P_t[y_{\mathrm{tr}}[i]] + \eta\, H_t[i]$\;
          $P_t[\hat c_t] \gets P_t[\hat c_t] - \eta\, H_t[i]$\;
        }
      }
    }
  }
}

\BlankLine
$(\alpha',\beta')\gets \arg\max\limits_{a\in\mathcal{A},\,b\in\mathcal{B}}
\mathrm{Acc}\!\left(y_{\mathrm{val}},\mathrm{Predict}(P_h,P_z,P_o,a,b;\mathrm{val})\right)$\;
$\hat y_{\mathrm{te}}^{(1)} \gets \mathrm{Predict}(P_h,P_z,P_o,\alpha',\beta';\mathrm{te})$\;

\Return{$\hat y_{\mathrm{te}}^{(0)},\hat y_{\mathrm{te}}^{(1)},(P_h,P_z,P_o),(\alpha,\beta),(\alpha',\beta')$}\;
\caption{Topology-guided HDC runner with late fusion and OnlineHD adaptation.}
\label{alg:topo_hdc}
\end{algorithm}


%% file: bibtex.bib
@article{smets2024encoding,
  title={An encoding framework for binarized images using hyperdimensional computing},
  author={Smets, Laura and Van Leekwijck, Werner and Tsang, Ing Jyh and Latr{\'e}, Steven},
  journal={Frontiers in big data},
  volume={7},
  pages={1371518},
  year={2024},
  publisher={Frontiers Media SA}
}

@inproceedings{verges2025learning,
  title={Learning encoding phasors with fractional power encoding},
  author={Verg{\'e}s, Pere and Nicolau, Alexandru and Givargis, Tony},
  booktitle={2025 11th International Conference on Computing and Artificial Intelligence (ICCAI)},
  pages={537--543},
  year={2025},
  organization={IEEE}
}

@article{kanerva2009hyperdimensional,
  title={Hyperdimensional computing: An introduction to computing in distributed representation with high-dimensional random vectors},
  author={Kanerva, Pentti},
  journal={Cognitive computation},
  volume={1},
  number={2},
  pages={139--159},
  year={2009},
  publisher={Springer}
}

@article{kleyko2023survey,
  title={A survey on hyperdimensional computing aka vector symbolic architectures, part ii: Applications, cognitive models, and challenges},
  author={Kleyko, Denis and Rachkovskij, Dmitri and Osipov, Evgeny and Rahimi, Abbas},
  journal={ACM Computing Surveys},
  volume={55},
  number={9},
  pages={1--52},
  year={2023},
  publisher={ACM New York, NY}
}

@inproceedings{rahimi2016robust,
  title={A robust and energy-efficient classifier using brain-inspired hyperdimensional computing},
  author={Rahimi, Abbas and Kanerva, Pentti and Rabaey, Jan M},
  booktitle={Proceedings of the 2016 international symposium on low power electronics and design},
  pages={64--69},
  year={2016}
}

@article{rosenblatt1958perceptron,
  title={The perceptron: a probabilistic model for information storage and organization in the brain.},
  author={Rosenblatt, Frank},
  journal={Psychological review},
  volume={65},
  number={6},
  pages={386},
  year={1958},
  publisher={American Psychological Association}
}

@inproceedings{nguyen2015deep,
  title={Deep neural networks are easily fooled: High confidence predictions for unrecognizable images},
  author={Nguyen, Anh and Yosinski, Jason and Clune, Jeff},
  booktitle={Proceedings of the IEEE conference on computer vision and pattern recognition},
  pages={427--436},
  year={2015}
}

@article{carlsson2020topological,
  title={Topological methods for data modelling},
  author={Carlsson, Gunnar},
  journal={Nature Reviews Physics},
  volume={2},
  number={12},
  pages={697--708},
  year={2020},
  publisher={Nature Publishing Group UK London}
}

@article{carlsson2008local,
  title={On the local behavior of spaces of natural images},
  author={Carlsson, Gunnar and Ishkhanov, Tigran and De Silva, Vin and Zomorodian, Afra},
  journal={International journal of computer vision},
  volume={76},
  number={1},
  pages={1--12},
  year={2008},
  publisher={Springer}
}

@article{naitzat2020topology,
  title={Topology of deep neural networks},
  author={Naitzat, Gregory and Zhitnikov, Andrey and Lim, Lek-Heng},
  journal={Journal of Machine Learning Research},
  volume={21},
  number={184},
  pages={1--40},
  year={2020}
}

@article{zhou2017exploring,
  title={Exploring generalized shape analysis by topological representations},
  author={Zhou, Zhen and Huang, Yongzhen and Wang, Liang and Tan, Tieniu},
  journal={Pattern Recognition Letters},
  volume={87},
  pages={177--185},
  year={2017},
  publisher={Elsevier}
}

@conference{edelsbrunner2008persistent,
  author = {Edelsbrunner, H. and Harer, J.},
  booktitle = {Surveys on discrete and computational geometry},
  organization = {Amer Mathematical Society},
  pages = 257,
  title = {{Persistent homology -— a survey}},
  volume = 453,
  year = 2008
}

@article{novotni2004shape,
  title={Shape retrieval using 3D Zernike descriptors},
  author={Novotni, Marcin and Klein, Reinhard},
  journal={Computer-Aided Design},
  volume={36},
  number={11},
  pages={1047--1062},
  year={2004},
  publisher={Elsevier}
}

@article{li2008complex,
  title={Complex Zernike moments features for shape-based image retrieval},
  author={Li, Shan and Lee, Moon-Chuen and Pun, Chi-Man},
  journal={IEEE Transactions on Systems, Man, and Cybernetics-Part A: Systems and Humans},
  volume={39},
  number={1},
  pages={227--237},
  year={2008},
  publisher={IEEE}
}

@article{lecun2002gradient,
  title={Gradient-based learning applied to document recognition},
  author={LeCun, Yann and Bottou, L{\'e}on and Bengio, Yoshua and Haffner, Patrick},
  journal={Proceedings of the IEEE},
  volume={86},
  number={11},
  pages={2278--2324},
  year={2002},
  publisher={Ieee}
}

@inproceedings{cohen2017emnist,
  title={EMNIST: Extending MNIST to handwritten letters},
  author={Cohen, Gregory and Afshar, Saeed and Tapson, Jonathan and Van Schaik, Andre},
  booktitle={2017 international joint conference on neural networks (IJCNN)},
  pages={2921--2926},
  year={2017},
  organization={IEEE}
}
